\documentclass[twoside]{article}

\usepackage{authblk}
\usepackage{style}
\usepackage[round]{natbib}

\usepackage{silence}
\usepackage[utf8]{inputenc} 
\usepackage[T1]{fontenc}    
\usepackage{url}            
\usepackage{booktabs}       
\usepackage{amsfonts}       
\usepackage{nicefrac}       
\usepackage{microtype}      
\usepackage{xcolor}         

\usepackage[toc,page,header]{appendix}
\usepackage{minitoc}

\usepackage{graphicx}
\usepackage{subcaption}
\usepackage{amsmath}
\usepackage{amssymb}
\usepackage{mathtools}
\usepackage{xspace}
\usepackage{placeins}
\usepackage{bm}

\usepackage{amsthm}
\usepackage{thmtools} 
\usepackage{thm-restate}
\theoremstyle{plain}

\newtheorem{lemma}{Lemma}

\theoremstyle{definition}

\newtheorem{assumption}{Assumption}
\theoremstyle{remark}

\usepackage{algorithm}
\usepackage{algorithmic}

\usepackage[hidelinks]{hyperref}
\hypersetup{
	pdfborder={0 0 0},
	colorlinks,
	citecolor={[rgb]{0.0, 0.0, 0.5}},
	linkcolor={[rgb]{0.0, 0.0, 0.5}},
	urlcolor=black
}

\usepackage[capitalize,noabbrev]{cleveref}
\crefname{assumption}{Assumption}{Assumptions}
\crefname{theorem}{Theorem}{Theorems}
\crefname{equation}{}{}
\crefname{ALC@unique}{Line}{Lines}

\newcommand{\multitaskGP}{\textsc{MTGP}\xspace}

\newcommand{\ABLR}{\textsc{ABLR}\xspace}
\newcommand{\RGPE}{\textsc{RGPE}\xspace}
\newcommand{\SGPT}{\textsc{SGPT}\xspace}
\newcommand{\StackGP}{\textsc{StackGP}\xspace}
\newcommand{\RMBO}{\textsc{RM-GP-UCB}\xspace}

\newcommand{\cov}[1]{\mathrm{Cov}\left(#1\right)}
\newcommand{\corr}[1]{\mathrm{Corr}\left(#1\right)}

\DeclareMathOperator*{\argmax}{arg\,max}

\newcommand{\diag}{\mathrm{diag}}

\newcommand{\method}{\mathrm{ScaML}}
\newcommand{\ourmethod}{\textsc{ScaML-GP}\xspace}
\newcommand{\x}{\ensuremath{\mathbf{x}}}
\newcommand{\X}{\ensuremath{\mathbf{X}}}

\newcommand{\test}{t}
\newcommand{\meta}{m}
\newcommand{\nummeta}{M}
\newcommand{\metaset}{\mathcal{M}}
\newcommand{\metadata}{\mathcal{D}_{1:\nummeta}}

\author[1]{Petru Tighineanu}
\author[1]{Lukas Grossberger}
\author[1]{Paul Baireuther}
\author[1]{Kathrin Skubch}
\author[1]{Stefan Falkner}
\author[1]{Julia Vinogradska}
\author[1]{Felix Berkenkamp}
\affil[1]{Bosch Center for Artificial Intelligence, Renningen, Germany}

\title{Scalable Meta-Learning with Gaussian Processes}

\date{}

\begin{document}

\maketitle
\thispagestyle{empty}

\doparttoc 
\faketableofcontents 

\begin{abstract}
  Meta-learning is a powerful approach that exploits historical data to quickly solve new tasks from the same distribution. In the low-data regime, methods based on the closed-form posterior of Gaussian processes (GP) together with Bayesian optimization have achieved high performance. However, these methods are either computationally expensive or introduce assumptions that hinder a principled propagation of uncertainty between task models. This may disrupt the balance between exploration and exploitation during optimization.
  In this paper, we develop \ourmethod, a modular GP model for meta-learning that is scalable in the number of tasks. Our core contribution is a carefully designed multi-task kernel that enables hierarchical training and task scalability. Conditioning \ourmethod on the meta-data exposes its modular nature yielding a test-task prior that combines the posteriors of meta-task GPs.
  In synthetic and real-world meta-learning experiments, we demonstrate that \ourmethod can learn efficiently both with few and many meta-tasks.
\end{abstract}

\section{Introduction}

Meta-learning improves a learning system's performance on a new task by leveraging data from similar tasks \citep{pmlr-v70-finn17a}. This powerful learning paradigm has enabled numerous new applications in optimization \citep{rothfuss2022metalearning}, reinforcement learning \citep{wang2016learning} and other domains \citep{hariharan2017low}.
While meta-learning approaches that build on neural networks are highly successful in the large data setting, probabilistic models that extract more information out of scarce data have an advantage in the low data regime. In particular, methods that combine probabilistic priors in the form of Gaussian process (GP) models \citep{rasmussengp} with Bayesian optimization (BO) \citep{BO_Shahriari} have been shown to achieve high performance in the small data setting \citep{tighineanu2022transfer}. Several important problems fall into this category of scarce data tasks, e.g., optimizing the hyperparameters (HPs) of machine learning models~\citep{snoek_hpo}, materials design~\citep{zhang2020}, or tuning controller gains~\citep{calandra2016bayesian}. These tasks have in common that obtaining new data on the test-task is costly, e.g., due to materials or compute cost, wear-and-tear of a physical system, or manual effort involved. At the same time, data from many similar tasks are typically available and can be leveraged to speed up learning.

In this setting, a joint Bayesian treatment of the test data and the data from the meta-tasks is beneficial, since it accounts for uncertainty across different task functions. This handling of uncertainty is crucial: while there may be a large meta-data set, the data for each of the meta-tasks is typically scarce so that neglecting the meta-task uncertainty often leads to overconfident models. To this end, several GP-based models have been introduced that model the joint distribution across tasks \citep{NIPS2013_swersky,EnvGP,pmlr-v54-shilton17a}. However, a full Bayesian treatment is typically computationally too costly. Instead, \citet{feurer2018rgpe,wistuba2018scalable} propose to use ensembles of GPs, which scale more favorably. However, they lack a joint Bayesian task model and hyperparameter inference is done through heuristics. 

\paragraph{Contribution} We present Scalable Meta-Learning with Gaussian Processes (\ourmethod), a modular GP for meta-learning that is scalable in the number of tasks. In contrast to previous GP models, we introduce assumptions on the correlation between meta- and test-tasks and show that these lead to a posterior model that scales linearly in the number of meta-tasks and can thus be learned efficiently. Our experiments indicate that \ourmethod outperforms existing methods in the low-data setting that we focus on.

\section{Related Work}
\label{sec:related_work}
Meta-learning involves learning how to speed up learning new tasks given data from similar tasks \citep{meta_learn_schmidhuber, bengio_1991}. There exists a large body of literature on meta-learning different key components of optimization: meta-learning the whole optimizer \citep{pmlr-v70-chen17e, li2017learning, grad_descent_by_grad_descent,metz2018learning}, the model \citep{pmlr-v70-finn17a,ABLR,flennerhag2018transferring,Wistuba2021few}, or the acquisition function \citep{Volpp2020Meta-Learning}. While these methods are powerful when data is abundant, they are not suited for the low-data regime of global black-box optimization of expensive functions.

In the low-data regime, GP-based approaches are dominant due to their sample efficiency. The uncertainty estimates of GP posteriors are highly informative for identifying interesting regions of the search space. This well-calibrated uncertainty estimate comes at the price of training a GP on the full meta- and test-data jointly. Approaches that follow this path \citep{NIPS2013_swersky,pmlr-v33-yogatama14,NIPS2017_MISO,EnvGP,pmlr-v54-shilton17a} perform favorably with little meta-data. However, their computational cost is cubic in the overall number of data points and quickly becomes prohibitive for common meta-learning scenarios. While several approximations scaling GPs to larger data sets exist \citep{Liu2020ScalableGPs} and can in principle be combined with the above approaches, none of these directly applies to the meta-learning scenario that we consider in this paper.
An interesting approach to achieve task scalability is learning a parametric GP prior on the meta-data \citep{Wang2021pre}.

Another explored avenue that ensures scalability is training GP models for each meta- and test-task, and combining their predictions to inform the search for the global optimum. 
\citet{Dai2022provably} fit per-task GPs and model the acquisition function as a weighted sum of task-based acquisition functions. This approach is scalable in the number of tasks but relies on a heuristic weighing of the relative importance of the meta- and test tasks.
\citet{golovin2017stackgp} build a hierarchical model where only the information about the mean prediction is propagated from the meta-tasks. 
\citet{feurer2018rgpe,wistuba2018scalable} build a GP ensemble with ranking-based weights for all task GPs. These methods account for some uncertainty propagation from the meta-tasks in a heuristic way. 
Our method combines the best of both worlds by being scalable in the number of meta-tasks \emph{and} having a well-calibrated uncertainty estimate by defining a joint probability distribution over all data.

\section{Problem Statement}
\label{sec:problem_statement}

We consider a meta-learning setup, where the task at test time is to efficiently maximize a function $f_\test \colon D \to \mathbb{R}$ with $f_\test$ sampled from some unknown distribution. To maximize $f_\test$, we sequentially evaluate parameters $\x_n$ to obtain noisy observations $y_n = f_\test(\x_n) + \omega_n$, where $ \omega_n \sim \mathcal{N}(0, \sigma_\test^2)$ is i.i.d.\ zero-mean Gaussian noise. To this end, we use the $N_\test$ previous observations $\mathcal{D}_\test = \{ \x_n, y_n \}_{n=1}^{N_\test}$ to build a probabilistic model for $f_\test$. Specifically, we focus on GPs \citep{rasmussengp}, $f_\test \sim \mathcal{GP}(m, k)$, which model function values via a joint Gaussian distribution that is parametrized through a prior mean function $m(\cdot)$ and a kernel $k(\cdot, \cdot)$. Conditioned on the data $\mathcal{D}_\test$, the posterior is another GP with mean and covariance at a query parameter $\x$ given by
\begin{equation}
\begin{aligned}
\mu_\test(\x) 
    &= m(\x) + k(\x, \X_\test) 
    \\
    &\times \left(k(\X_\test, \X_\test) + \sigma_\test^2 \mathbf{I}\right)^{-1} 
    \left( \mathbf{y}_\test - m(\x) \right), \\
\Sigma_\test(\x, \x') 
	&=  k(\x, \x') - k(\x, \X_\test ) 
    \\
    &\times \left(k(\X_\test, \X_\test) +\sigma_\test^2\mathbf{I}\right)^{-1} k(\X_\test, \x'),
\end{aligned}
\label{eq:test_posterior}
\end{equation}
where $\X_\test = (\x_1, \dots, \x_{N_\test})$ and $\mathbf{y}_\test = (y_1, \dots, y_{N_\test})$ is the vector of noisy observations. To improve the model we assume access to meta-data from $\nummeta$ related tasks $\meta \in \metaset = \{1, \dots, \nummeta\}$ that come from the same distribution of tasks. For each of those we have access to a dataset $\mathcal{D}_\meta = \{ \x_{\meta, n}, y_{\meta, n} \}_{n=1}^{N_\meta}$ that is based on $N_\meta$ noisy observations $y_{\meta, n} = f_\meta(\x_{\meta, n}) + \omega_{\meta, n}$ of the corresponding task corrupted by different noise levels $\omega_{\meta, n} \sim \mathcal{N}(0, \sigma_\meta^2)$. Together, these datasets form the meta-data $\metadata = \cup_{\meta \in \metaset} \mathcal{D}_\meta$.
This meta-data can be incorporated in the GP by considering a \emph{joint} model over the meta- and test-tasks. Multi-task GP (\multitaskGP) models are defined through an extended kernel that additionally models similarities between tasks,
\begin{equation}
\label{eq:multi_task_kernel}
    k(( \x, \nu), ( \x', \nu')) = \sum_{\meta \in \metaset \cup \{ \test \} } [\mathbf{W}_\meta]_{(\nu, \nu')} \, k_\meta( \x, \x') ,
\end{equation}
where $k_\meta$ are arbitrary kernel functions and $\mathbf{W}_\meta$ are positive semi-definite matrices called \emph{coregionalization matrices} whose entries $[\mathbf{W}_\meta]_{(\nu, \nu')}$ model the covariance between two tasks $\nu$ and $\nu'$ \citep{alvarez2012kernels}. By conditioning this joint model on both $\metadata$ and $\mathcal{D}_\test$ we obtain a tighter posterior on $f_\test$ through the typical GP equations in \cref{eq:test_posterior}. However, these models are computationally expensive with a complexity that scales cubically in the number of all data points, and are difficult to train in practice due to the large number of HPs, which scale at best quadratic in the number of tasks~\citep{bonilla2008multitask,tighineanu2022transfer}. For notational convenience, we index the test-task as the $\nummeta+1$th task, $\test = \nummeta + 1$, in the following.

Given a prior over the function $f_\test$, BO methods use the posterior in \cref{eq:test_posterior} to sequentially query a new parameter $\x_{N_\test + 1}$ that is informative about the optimum of $f_\test$ by solving an auxiliary optimization problem based on an acquisition function~$\alpha$,
\begin{equation}
    \x_{N_\test + 1} = \argmax_{\x \in D} \alpha( f_\test \mid \x, \mathcal{D}_\test, \metadata).
    \label{eq:bo}
\end{equation}
Several choices for $\alpha$ exist in the literature \citep{jones1998efficient,srinivas10gaussian}. The performance of these methods commonly depend on the quality of the test-prior, which is the focus of this paper.

\section{Scalable Multi-Task Kernel}
\label{sec:our_approach}

In this section, we present \ourmethod, a modular meta-learning model that is efficient to train and evaluate. We introduce two assumptions on the \multitaskGP model in \cref{eq:multi_task_kernel} that restrict learning only to the most relevant covariances and lead to a modular GP posterior that can be evaluated efficiently. The first assumption neglects the correlations between meta-tasks. The number of these correlations scales quadratically with the number of meta-tasks and learning them is challenging in the regime of many meta-tasks with scarce data. For convenience, we write $\cov{f_\meta, f_{\meta'}} = c$ for some $c \geq 0$ to focus on covariance between meta-tasks, instead of the more explicit $\cov{f_\meta(\x), f_{\meta'}(\x')} = c\, k_\meta(\x, \x')$.

\begin{assumption}
    ${ \cov{f_{1:\nummeta}, f_{1:\nummeta}} = \mathbf{I} }$. 
    \label{as:independent_sources}
\end{assumption}

While \cref{as:independent_sources} is usually violated in practice, this does not prohibit learning as long the amount of data per meta-task is sufficient to allow for a probabilistic description of the meta-task function. This is usually the case in many real-world applications. Note that meta-tasks are uncorrelated only in their prior distribution: since each meta-task can affect the test-task, conditioning on the test-data $\mathcal{D}_\test$ induces correlations between meta-tasks via the explaining-away effect. The assumption that $\cov{f_\meta, f_\meta} = 1$ ensures that each kernel $k_\meta$ in \cref{eq:multi_task_kernel} models the marginal distribution of the corresponding meta-task $f_\meta$. Together, these two properties are critical to make our model efficient to learn for large number of meta-tasks. 

With the second assumption we restrict the test-task model to be additive in the marginal models of each meta-task. While additive models have been considered before, e.g., by \citet{duvenaud2011additive,Marco17VirtualvsReal}, these have been framed as $f_\test$ being a direct sum of meta-task models. In contrast, we require the weaker assumption that $f_\test$ is additive in functions that (anti-)correlate perfectly with the meta-task models. As for the covariance, we introduce the short-hand notation $\corr{f_\meta, f_{\meta'}} = \corr{f_\meta(\x), f_{\meta'}(\x)}$.

\begin{assumption}
The test-task model can be written as ${ f_\test = \tilde{f}_\test + \sum_{\meta \in \metaset} \tilde{f}_\meta }$, with ${ \big| \mathrm{Corr} \big( \tilde{f}_\meta, f_\meta \big) \big|= 1 }$, $\mathrm{Cov} \big(\tilde f_\test, f_\test\big) = 1$ and $\mathrm{Cov} \big( \tilde{f}_\test, f_\meta \big) = 0$ for all $\meta \in \metaset$.
    \label{as:strong_correlation}
\end{assumption}
By restraining the components of the test-task $\tilde{f}_\meta$ and meta-task $f_\meta$ models to correlate perfectly, we directly model the intuition that parts of the meta-task functions should be reflected in the test-task. Only the scale of these functions remains a free parameter that is learned. 
The residual model $\tilde{f}_\test$ is independent of the meta-tasks and models parts of the test-task that cannot be explained by the meta-task models. Together, \cref{as:independent_sources,as:strong_correlation} enforce structure on the coregionalization matrices in \cref{eq:multi_task_kernel}. \Cref{as:independent_sources} enforces $[\mathbf{W}_\meta]_{(\meta, \meta')}$ to be zero when $\meta \neq \meta'$, while \cref{as:strong_correlation} enforces $[\mathbf{W}_\meta]_{(\meta, \meta)} = 1$, so that the variation of each meta-task is modeled directly by the corresponding kernel $k_\meta$. The assumption about the correlation additionally leads to matrices $\mathbf{W}_\meta$ that are parametrized by an unconstrained scalar parameter $w_\meta \in\mathbb R$ for each meta-task $\meta \in \metaset$. Concretely, the matrices have all zero entries except
\begin{equation}
    \begin{aligned}
    [\mathbf W_\test]_{(\test, \test)} &= [\mathbf W_\meta]_{(\meta, \meta)} = 1, \\
    [\mathbf W_\meta]_{(\meta, \test)} &= [\mathbf W_\meta]_{(\test, \meta)} = w_\meta, \\
     [\mathbf W_\meta]_{(\test, \test)} &= w_\meta^2.
    \end{aligned}
\label{eq:correg_matrices} 
\end{equation}
As an example, for one meta-task $\nummeta = 1$ we have
\begin{equation*}
    \mathbf{W}_1 = \begin{bmatrix}
    1 & w_1 \\
    w_1 & w_1^2
    \end{bmatrix},
    \quad
    \mathbf{W}_\test = \begin{bmatrix}
    0 & 0 \\
    0 & 1
    \end{bmatrix}.
\end{equation*}
It is easy to verify that both matrices are positive semi-definite (see \cref{ap:kernel_properties}) and that we have $ | \mathrm{Corr} ( \tilde{f}_1, f_1 ) | = | w_1  / \sqrt{1^2 \times w_1^2 } | = 1 $. Thus, while we constrain the meta- and test functions to be perfectly correlated, the magnitude of $w_\meta$ determines to what extent the meta-task is relevant for the test-task: 
The prior for $f_\test$ is $w_1^2 k_\meta(\cdot, \cdot) + k_\test(\cdot, \cdot)$ and in the limit ${ w_1 \to 0 }$ they are modeled as being independent. The same reasoning holds for multiple tasks.

\begin{restatable}{lemma}{jointkernel}
\label{lem:valid_joint_kernel}
\cref{as:independent_sources,as:strong_correlation} with $w_\meta \in \mathbb{R}$ for $\meta \in \metaset$ yield a valid multi-task kernel given by
\begin{equation}
\begin{aligned}
k_\method^\mathrm{joint}((\x, \nu), (\x', \nu')) &= 
  \delta_{\nu=t} \delta_{\nu'=t} k_\test(\x, \x') \\
  &+ \sum_{\meta \in \metaset} g_\meta(\nu) g_\meta(\nu') k_\meta(\x, \x'),
\end{aligned}
\label{eq:our_joint_kernel}
\end{equation}
where $g_\meta(\nu)$ is equal to $w_\meta$ if $\nu = \test$, one if $\nu = \meta$, and zero otherwise. 
$\delta_{i=j}$ is the Dirac-delta.
\end{restatable}

Based on \cref{lem:valid_joint_kernel}, we have a valid joint kernel over meta- and test-tasks that is parameterized by the scalars $w_\meta \in \mathbb{R}$. This successfully limits the number of parameters to scale linearly in the number of meta-tasks. The scalable and modular nature of \ourmethod is revealed by conditioning \cref{eq:our_joint_kernel} on the meta-data.

\begin{restatable}{theorem}{testprior}
    \label{thm:test_prior}
    Under a zero-mean GP prior with multi-task kernel given by \cref{eq:our_joint_kernel}, the test-task distribution conditioned on the meta-data is given by ${ f_\test\mid \metadata \sim \mathcal{GP}( m_\method, \Sigma_\method ) }$ with
	\begin{equation}
    \begin{aligned}
	   m_\method(\x) &= \sum_{\meta \in \metaset} w_\meta \mu_\meta(\x), \\
	   k_\method(\x, \x') &= k_\test(\x, \x') + \sum_{\meta \in \metaset} w_\meta^2 \Sigma_\meta(\x, \x'),
    \end{aligned}
	\label{eq:test_prior}
	\end{equation}
	where $\mu_\meta(\x)$ and $\Sigma_\meta(\x, \x')$ are the per-task posterior mean and covariance conditioned on $\mathcal{D}_\meta$.
\end{restatable}

\begin{figure*}[t]
\includegraphics[width=\textwidth]{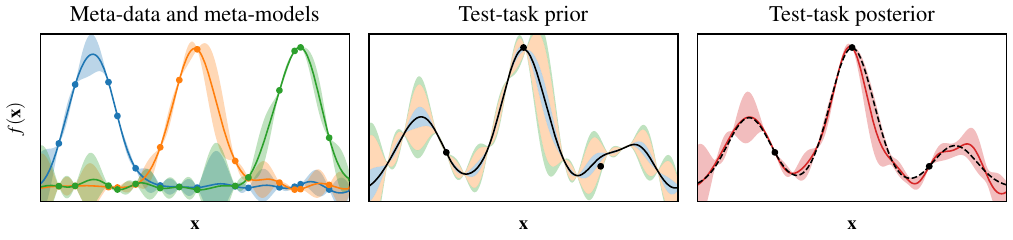}
\caption{Illustration of \ourmethod. We train individual GPs per meta-task (left), which combine with the task kernel $k_\test$ and the task-weights $w_\meta$ to form the test-task prior. The HPs are inferred via MAP-inference to obtain the figure in the middle, where the shaded colors correspond to the meta-tasks contribution to the standard deviation. The posterior can then be obtained by conditioning the test-task prior distribution on $\mathcal{D}_\test$ via \cref{eq:test_posterior} (right). Notice the decreased uncertainty estimate of the test posterior at the location of the meta-data.}
\label{fig:method_illustration}
\end{figure*}

Following \cref{thm:test_prior} we can model each meta-task $\meta$ with an individual GP based on a kernel $k_\meta$. The resulting prior distribution of the test-task $f_\test$ is a GP given by the weighted sum of meta-task posteriors according to \cref{eq:test_prior}. 
\ourmethod thus models a full joint distribution over tasks and yields a prior on the test function that can be conditioned on $\mathcal{D}_\test$ through \cref{eq:test_posterior} in order to obtain the test-task posterior 
\begin{equation}
p(f_\test \mid \x, \mathcal{D}_\meta, \mathcal{D}_\test) = \mathcal{N}\left( \mu_\test(x), \Sigma_\test(\x,\x)\right)
\label{eq:test_posterior}
\end{equation}
based on the prior mean and kernel from \cref{thm:test_prior}. Next to enabling a full Bayesian treatment of uncertainty, this also allows us to determine the meta-task weights $w_\meta$ by maximizing the likelihood. In light of \cref{thm:test_prior}, \ourmethod reduces the complexity of the original \multitaskGP from cubic in the total number of points to \emph{linear} in the number of tasks. We achieve this solely by enforcing \cref{as:independent_sources,as:strong_correlation} and without introducing numerical approximations. We illustrate the inner workings of \ourmethod in \cref{fig:method_illustration}.

\paragraph{Likelihood optimization}
So far we have assumed the HPs $\theta_\meta$ of the meta-task kernels $k_\meta$, and the test-task HPs $\theta_\test$, which contain the parameters of $k_\test$ and the weights $w_\meta$, to be given. In practice, they are inferred from the data $\metadata$ and $\mathcal{D}_\test$. Naive evaluation of the likelihood of the joint task model in \cref{eq:our_joint_kernel} is expensive, $O((N_\test + \nummeta \overline{N}_\meta)^3)$ with $\overline{N}_\meta = \max_{\meta \in \metaset} N_\meta$, since it depends on all data. However, any model that complies with \cref{as:independent_sources} is scalable in $\nummeta$ since
\begin{align}
    \label{eq:full_likelihood_decomposition}
    &\log p\left( \mathbf{y}_\test, \mathbf{y}_{1:\nummeta} \mid \X_\test, \X_{1:\nummeta}, \theta_\test, \theta_{1:\nummeta} \right) 
    = \\
    &\log p\left( \mathbf{y}_\test \mid \metadata, \X_\test, \theta_\test, \theta_{1:\nummeta} \right) 
    + \hspace{-0.5em} 
    \sum_{\meta \in \metaset}
    \hspace{-0.4em}
    \log p\left( \mathbf{y}_\meta \mid \X_\meta, \theta_\meta \right) \notag
    .
\end{align}
The second term is the per-meta-task likelihood that can be computed at cost $O( \nummeta N_\meta^3 )$, while the first is the likelihood under the test-task prior given by \cref{thm:test_prior}. Given the already inverted meta-task kernel matrices, computing the posterior meta-task covariances at test-task points $\X_\test$ is of complexity $O( \nummeta ( N_\test^2 \overline{N}_\meta + N_\test \overline{N}_\meta^2))$. Together with the resulting test-task likelihood, $O(N_\test^3)$, this yields a total complexity of $O( \nummeta ( \overline{N}_\meta^3 + N_\test^2 \overline{N}_\meta + N_\test \overline{N}_\meta^2) + N_\test^3 )$, which is linear in the number of meta-tasks $\nummeta$ and thus enables scalable optimization.

In practice, the number of test parameters in $\X_\test$ is usually smaller than the available meta-data since the meta-prior already contains significant information. This leads to a weak dependence between the meta-model parameters $\theta_\meta$ and the test data $\mathbf{y}_\test$. This is especially true for \ourmethod, since the marginal per-task model only depends on $k_\meta$ and is thus independent of $\theta_\test$. We therefore suggest to modularize $\ourmethod$ by assuming conditional independence between $\theta_\meta$ and $\mathcal{D}_\test$ \citep{bayarri2009modularization}:

\begin{assumption}
    For all meta-tasks ${ \meta \in \metaset }$, we have ${ p\left( \theta_\meta \mid \mathcal{D}_\meta, \mathcal{D}_\test \right) = p\left( \theta_\meta \mid \mathcal{D}_\meta \right) }$.
    \label{as:likelihood_modularization}
\end{assumption}

\cref{as:likelihood_modularization} allows us to infer the meta-task HPs $\theta_{1:\nummeta}$ independently of the test-task HPs $\theta_\test$. Thus, we can optimize the meta-task GPs in parallel based only on their individual data, $\theta_\meta^\star = \argmax_{\theta_\meta} \log p(\mathbf{y}_\meta \mid \X_\meta, \theta_\meta)$. Afterwards, we compute and cache the meta-task GP posterior mean $\mu_\meta(\X_\test)$ and covariance matrix $\Sigma_\meta(\X_\test, \X_\test)$. Since the test-task prior in \cref{eq:test_prior} depends on these quantities and not  on $\theta_\test$ due to \cref{as:likelihood_modularization}, we can optimize the test-task likelihood,
\begin{equation}
    \theta_\test^\star = \argmax_{\theta_\test} \log p\left( \mathbf{y}_\test \mid \metadata, \X_\test, \theta_\test, \theta_{1:\nummeta}^\star \right),
    \label{eq:test_likelihood}
\end{equation}
at cost of only $O(\nummeta N_\test^2 + N_\test^3)$, which is cheap to evaluate. 
Together, this enables scalable meta-learning with Gaussian processes (\ourmethod) and we use this simplification in our experiments. We summarize the algorithm in \cref{alg:overview}.

\begin{algorithm}[tb]
   \caption{\ourmethod}
   \label{alg:overview}
\begin{algorithmic}[1]
   \STATE {\bfseries Input:} meta-data $\metadata = \cup_{\meta \in \metaset} \mathcal{D}_\meta$ \label{alg:overview:input}
   \STATE Train individual GP models per meta-task and optimize $\theta_\meta$ \label{alg:overview:meta-gps}
   \STATE Construct the test-task prior as in \cref{eq:test_prior}, and cache $\mu_\meta(\mathbf{X}_\test)$ and $\Sigma_\meta(\mathbf{X}_\test, \mathbf{X}_\test)$. \label{alg:overview:test-prior}
   \STATE Optimize the test-task HPs $\theta_\test$ as in \cref{eq:test_likelihood} \label{alg:overview:test-likelihood}
   \STATE Condition the prior on $\mathcal{D}_\test$ as in \cref{eq:test_posterior} to obtain the posterior distribution for $f_\test$ \label{alg:overview:test-posterior}
\end{algorithmic}
\end{algorithm}

\paragraph{Discussion and limitations}
\cref{thm:test_prior} provides a scalable and structured way to distill meta-information into a prior for a test-task GP. The key component for this is \cref{as:strong_correlation}, which assumes an additive model. While these models reflect many real-world situations, more flexible meta-learning models based on neural networks are in principle able to learn more complex relationships between the meta- and test-tasks. However, by relaxing the model assumptions these methods also require significantly more data. Thus \ourmethod is most suited when the amount of data per task is relatively scarce. While the overall approach scales linearly with the number of meta-tasks and enables parallel optimization, each model is still a standard Gaussian process, which scales cubically in the number of points per task. For large number of data points per task, $N_\meta$ or $N_\test$, the modular nature of \ourmethod allows for full scalability in the number of points by employing scalable approximations for the task GPs \citep{lazaro2010sparse}. While an interesting direction, we focus on closed-form inference here and leave the exploration of full scalability for future work.

\section{Experiments}
\label{sec:experiments}

\begin{figure*}[t]
    \centering
    \includegraphics[width=\textwidth]{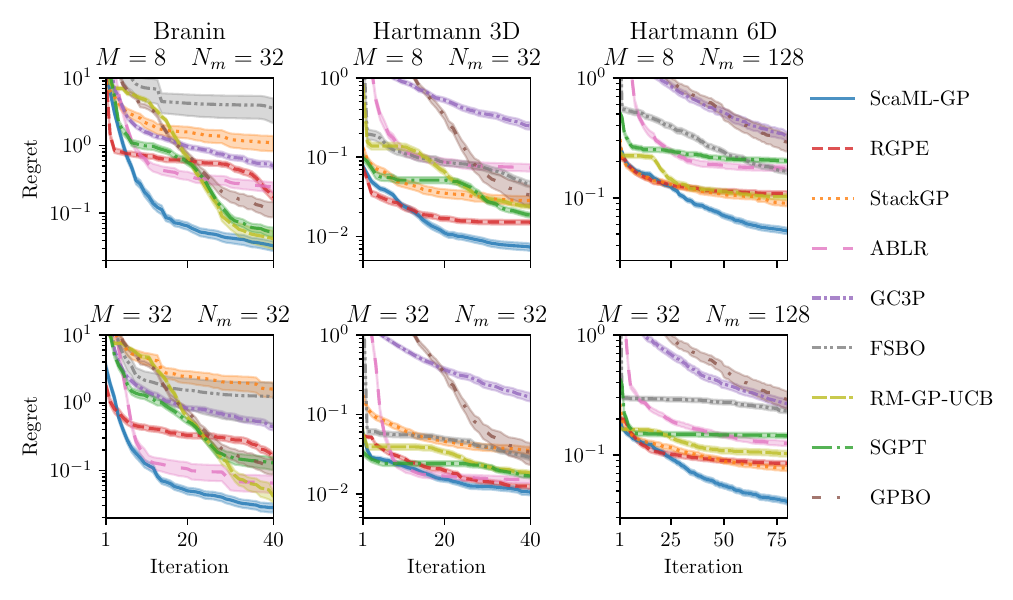}
    \caption{
        Experimental results on the synthetic benchmarks for two different meta-data configurations: $\nummeta = 8$ and $32$ meta-tasks in the top and bottom row, respectively. The meta-data are sampled uniformly at random and contain $N_\meta=32$ points per meta-task for Branin and Hartmann 3D, and 128 for Hartmann 6D, respectively. \ourmethod consistently achieves the lowest simple regret across all tasks. We provide a more detailed analysis on the effect of $M$ and $N_\meta$ on the performance in \cref{fig:ablation_summary}.
        }
    \label{fig:synthetic}
\end{figure*}

We evaluate our method (\ourmethod) on several optimization problems against the following baselines: standard GP-based BO (GPBO) without any meta-data, \RGPE \citep{feurer2018rgpe} and \SGPT~\citep{wistuba2018scalable}, both of which model the test function as a weighted sum of the predictions of all task GPs, the method by \citet{golovin2017stackgp}, which we dub \StackGP, since it trains a vertical stack of GPs and the posterior mean function of each GP is used as the prior mean function of the next GP, \ABLR~\citep{ABLR}, where Bayesian linear regression is performed on the deep features learned from the meta-data, \RMBO~\citep{Dai2022provably}, where the acquisition function is a weighted sum of individual acquisition functions of each model (see \cref{ap:implementation_details_rmbo} for implementation details), FSBO \citep{Wistuba2021few}, which is an extension of MAML to GPs, and GC3P \citep{salinas2020quantile} in which a Gaussian Copula regression
with a parametric prior are used to scale to large
data. We implement the models and run BO using BoTorch \citep{balandat2020botorch} and provide details on the experimental setup in \cref{ap:experiment_details}. 

\subsection{Synthetic Benchmarks}
\label{sec:synthetic_benchmarks}

Our synthetic-benchmark consists of conventional benchmark functions, where we place priors on some of their parameters. We consider the two-dimensional Branin, three-dimensional Hartmann 3D, and six-dimensional Hartmann 6D and provide details on the function families and the priors in \cref{ap:synthetic_benchmarks}. Note that these are generic meta-learning benchmarks that are not designed to fulfill \cref{as:independent_sources,as:strong_correlation,as:likelihood_modularization}.
We evaluate performance based on the simple regret ${ r = \max_{\x \in D} f_\test(x) - \max_{n \leq N_\test } f_\test(\x_n) \geq 0}$, which is the difference between the true optimum and the best function value obtained during the current optimization run. For each benchmark, we conduct 128 independent runs and report mean and standard error in our figures. For each run, the meta-data is consistent across all baselines for comparability.

The performance of the different baselines on the synthetic benchmarks for $M=8$ (top row) and $M=32$ (bottom row) meta-tasks is visualized in \cref{fig:synthetic}. The regret of GPBO converges in about 40 iterations for the two- and three-dimensional benchmarks and needs more than 80 iterations for Hartmann 6D. As expected, most meta-learning baselines converge faster than GPBO, since they can leverage additional information. In general, different baselines excel in different data-regimes. For instance, \RGPE is among the best for the Hartmann families, while \ABLR provides competitive performance on Branin after about ten iterations. 
In contrast, \ourmethod, consistently demonstrates fast task adaptation and achieves the lowest regret at the end of each experiment compared to all baseline methods. This demonstrates that there is an advantage to employing a joint Bayesian model across tasks. For \ABLR, we can see that performance improves significantly as we increase the number of meta-tasks, and thus the overall amount of meta-data, in the bottom row. In contrast, \ourmethod performs consistently across both domains, which aligns with the findings of \citet{tighineanu2022transfer} that GP-based methods have an advantage when meta-data is scarce.

\begin{figure*}[t]
    \centering
    \includegraphics[width=\textwidth]{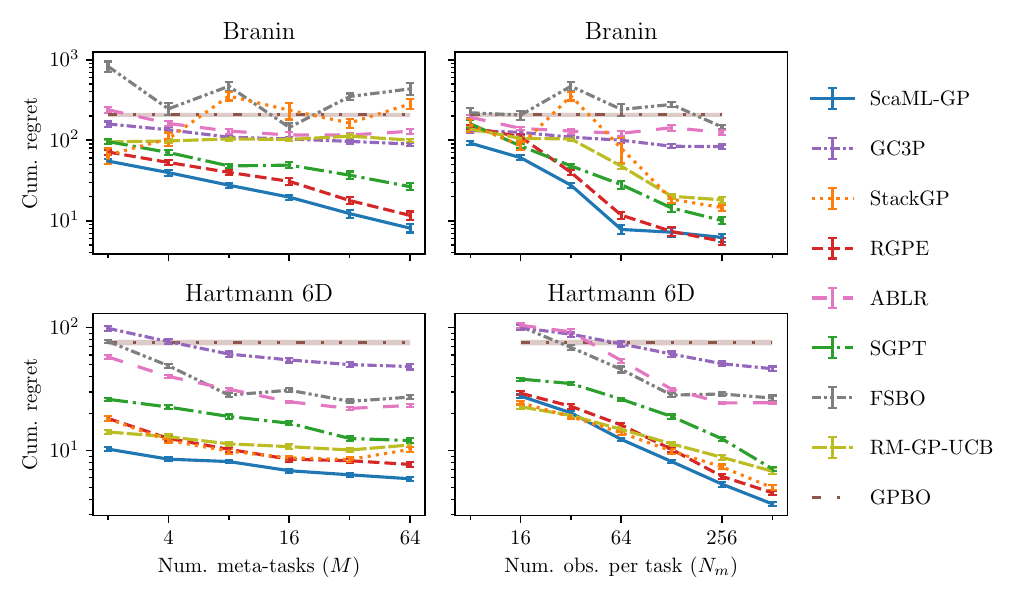}
    \caption{
        Summary of the performance of the meta-learning baselines on two synthetic benchmarks, Branin (top) and Hartmann 6D (bottom), as a function of the number of meta-tasks $\nummeta$ for a fixed number of points per task ($N_\meta=32$ and 128 for Branin and Hartmann 6D, respectively) on the left side, and as a function of the number of points per task $N_\meta$ for a fixed number of $\nummeta=8$ meta-tasks on the right side. Each data point denotes the mean and standard error of the cumulative regret at the end of the optimization. 
        \ourmethod consistently achieves the lowest cumulative regret across all data domains, which indicates that it is well suited for the low-data domain, but can also effectively scale to many tasks.
        }
    \label{fig:ablation_summary}
\end{figure*}

\paragraph{Ablation on meta-data}
We now study in detail how the performance of all methods depends on the amount of meta-data. In particular, we independently vary the amount of meta-data per task and the number of meta-tasks on the Branin and Hartmann 6D tasks. As we increase the number of tasks we get better coverage of the task distribution, whereas increasing the number of observations per task allows us to learn more about each individual meta-task. We present a condensed version of this ablation study in \cref{fig:ablation_summary}, where we plot the cumulative regret at the end of the optimization. More information including the simple regret plots are available in \cref{ap:ablation}. As in \cref{fig:synthetic}, we observe that most meta-learning baselines outperform standard GPBO throughout the ablation range. Similarly, as we increase the amount of meta-data, meta-learning methods generally improve performance since they have more information about the task family. The exception to this rule is \StackGP, which unlike other methods does not scale to many tasks $\nummeta$ (top left), since it builds a hierarchical sequential model based on mean functions only and is thus not able to convey uncertainty information effectively. This effect is most pronounced for Branin, where the the optimum of different tasks varies continuously, while for Hartmann 6D they are restricted to four discrete locations. For Branin we can observe in the top-right figure that starting at about a hundred observations per task performance starts to saturate, since that is sufficient information to train confident meta-task models. Overall \ourmethod outperforms other baselines across all data domains. While other methods can match the performance in some settings, \ourmethod performs consistently well across all settings.

\subsection{HPO Benchmarks}
\label{sec:hpo_benchmarks}

We compare all methods on a set of tasks where we have to optimize the HPs of machine-learning models. Each benchmark considers a specific machine learning model and a task consists of adapting the model's HPs in order to minimize the expected loss on a given dataset across several random seeds. To facilitate fast exploration, we base our study on the tabular benchmarks available as part of the HPOBench library \citep{eggensperger2021hpobench}. For each task, it provides a lookup table that maps HP configurations to the corresponding loss. For BO, we restrict the parameter space of each benchmark to contain only the discrete values in the lookup table. We consider HP optimization tasks for a support-vector-machine (SVM) model, logistic regression (LR), XGBoost model (XGB), random forest (RF), and neural-network (MLP) model.
We randomly assign one of the tasks to be our test function for a single independent run, while the remaining tasks are used for meta-learning. We sub-sample the meta-data to $N_\meta=64$ points per task for the two-dimensional SVM and LR, and to $N_\meta=128$ points per task for the four-dimensional XGB and RF, and five-dimensional MLP benchmarks. In addition, we provide results on the tabular FC-Net benchmarks Slice Localization, Protein Structure, Naval Propulsion, and Parkinson's Telemonitoring \citep{eggensperger2021hpobench}. For each of the four, we use one as test task and the other three as meta-tasks. We report the mean simple regret together with the standard error of the mean computed over 256 independent runs. We provide further details about the HPO benchmarks in \cref{ap:bo_on_tabular_benchmarks}.

\begin{figure*}[t]
    \centering
    \includegraphics[width=\textwidth]{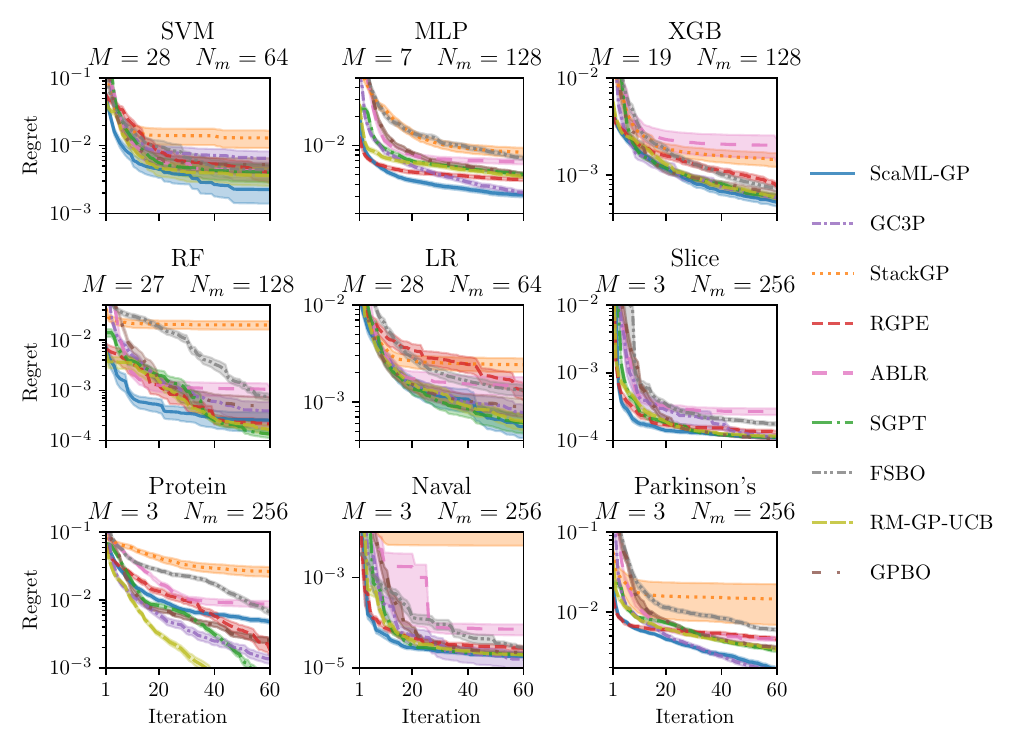}
    \caption{
        Experimental results on different hyperparameter optimization problems for machine learning models from five tabular HPO benchmarks (SVN, MLP, XGB, RF, and LR), as well four tabular FC-Net benchmarks (Slice, Protein, Naval, and Parkinson's). For each model, training runs on different datasets are used as meta-data with $N_\meta$ points per each of the $\nummeta$ meta-tasks to better inform hyperparameter choices on a new test-task dataset. \ourmethod is either the best or joint-best baseline on all benchmarks except the FC-Net Protein.
    }
    \label{fig:hpo}
\end{figure*}

In \cref{fig:hpo} we show a comparison of the methods on all nine HPO benchmarks.
Most meta-learning baselines successfully leverage the meta-data and achieve smaller regrets than GPBO early in the optimization. Towards the end of the optimization process, some baselines perform \emph{worse} than GPBO, which is most likely related to the inductive bias in the meta-data. Crucially, the threshold from better to worse varies among benchmarks, e.g., about ten iterations for XGB and thirty for RF. In practice, this threshold is not known to the user and makes it challenging to estimate the iteration budget or be confident in superior performance for a fixed budget. 
In contrast, \ourmethod is consistently competitive across all benchmarks in all optimization stages. This evidence is in line with the performance on the synthetic benchmarks in \cref{fig:synthetic}. The exception where ScaML-GP performs subpar is the FC-Net Protein benchmark. Our analysis indicates that this is because \cref{as:strong_correlation} is grossly violated such that no linear combination of the three meta-task models yields a useful test-task prior, see \cref{ap:validity_of_scamlgp_assumptions} for a detailed discussion.  We believe that the strong overall performance of \ourmethod is a result of the principled Bayesian approach together with assumptions that enable task scalability.

\section{Conclusion}

We have presented \ourmethod, a scalable GP for meta-learning. It is a specific instance of a multi-task GP model in \cref{eq:multi_task_kernel}, but based on explicit assumptions on the model structure in order to obtain a method that is scalable in the number of meta-tasks $\nummeta$. In particular, our test-prior is a weighted linear combination of the per-meta-task GP posterior distributions together with a residual model that accounts for test-functions that cannot be explained through the meta-data. In contrast to ensemble methods like \RGPE, this joint Bayesian model allows us to use maximum likelihood optimization in order to determine the weights. Moreover, we showed that hyperparameter inference in \ourmethod can be parallelized in order to enable efficient learning. Finally, we compared our method against a set of GP-based and neural-network-based meta-learning methods. \ourmethod consistently performed well across various benchmarks and number of meta-tasks.

\clearpage

\bibliography{references}
\bibliographystyle{plainnat}

\clearpage
\appendix
\thispagestyle{empty}
\onecolumn

\part{Overview} 

In the appendix we provide the detailed proofs for all claims in the paper, complexity analysis, ablation studies, and details on experiments. An overview is shown below.

\noptcrule
\parttoc 
\clearpage

\section{Kernel Properties}
\label{ap:kernel_properties}

We start by showing that for a single meta-task, the coregionalization is indeed positive semi-definite.

\begin{lemma}
\label{lem:ws_positive_semidefinite}
For any $w_\meta \in \mathbb{R}$ 
\begin{equation}
\mathbf{W}_\meta = \begin{bmatrix}
    1 & w_\meta \\
    w_\meta & w_\meta^2
    \end{bmatrix}
\end{equation}
is positive semi-definite.
\end{lemma}
\begin{proof}
Let $\x = (x_1, x_2)\in\mathbb R^2$, then $\x^T \mathbf{W}_\meta \x = x_1^2 + 2 w_\meta x_1 x_2 + w_\meta^2 x_2^2 = (x_1 + w_\meta x_2)^2 \geq 0$, i.e. $\mathbf W_\meta$ is positive semi-definite. \qedhere
\end{proof}

\begin{lemma}
For any $w_\meta\in\mathbb R$ the coregionalization matrices specified by \cref{eq:correg_matrices} are positive semi-definite.
\label{lem:coreg_matrices_psd}
\end{lemma}
\begin{proof}
Each matrix $\mathbf{W}_\meta$ has the same form as in \cref{lem:ws_positive_semidefinite}, but with additional zero rows and columns added.
That is, from \cref{lem:ws_positive_semidefinite} for any $\x= (\x_1, \dots, \x_{\nummeta + 1})\in\mathbb R^{\nummeta+1}$ we have 
\begin{align}
\x^T\mathbf W_\meta \x 
&= (x_\meta, x_{\nummeta+1})^T\begin{bmatrix}
    1 & w_\meta \\
    w_\meta & w_\meta^2
    \end{bmatrix} (x_\meta, x_{\nummeta+1})\geq 0.
\end{align}

According to \cref{eq:correg_matrices}, we have
\begin{equation*}
\mathbf{W}_\test = \begin{bmatrix} \mathbf{0} & \mathbf{0} \\ \mathbf{0} & 1 \end{bmatrix}\in \mathbb{R}^{(\nummeta + 1) \times (\nummeta + 1)},
\end{equation*}
from which we obtain 
$\x^\mathrm{T} \mathbf{W}_\test \x = \x_{\nummeta + 1}^2 \geq 0$
for any $\x= (\x_1, \dots, \x_{\nummeta + 1}) \in \mathbb{R}^{\nummeta + 1}$. \qedhere
\end{proof}

\jointkernel*
\begin{proof}

We begin by showing that the coregionalization matrices in \cref{eq:correg_matrices} are uniquely defined by  \cref{as:independent_sources,as:strong_correlation}. To see this, we first collect all terms in \cref{eq:multi_task_kernel}: 
\begin{align}
 k(( \x, \meta), ( \x', \meta)) 
  &= \cov{f_\meta(x), f_\meta(x')} 
   = k_\meta(x, x'), \quad \text{(\cref{as:independent_sources})} \\
 k(( \x, \meta), ( \x', \meta' \neq \meta))
  &= \cov{f_\meta(x), f_{\meta'}(x')}  
   = 0, \quad \text{(\cref{as:independent_sources})} \\
 k(( \x, \meta), ( \x', t))  
  &=\cov{f_\meta(x), f_\test(x')}\nonumber \\
  &=\cov{f_\meta(\x), \tilde f_\test + \sum_{\meta' \in \metaset} \tilde f_{\meta'}(\x')} \quad \text{(\cref{as:strong_correlation})} \nonumber \\
  &=\cov{f_\meta(\x), \sum_{\meta' \in \metaset} \tilde f_{\meta'}(\x')}  \quad \text{(\cref{as:strong_correlation})} \nonumber
\end{align}
The perfect (anti)correlation $\corr{f_\meta(\x), \tilde f_\meta(\x)}= \pm 1$  in \cref{as:strong_correlation} implies that $\tilde f_\meta(\x) = w_\meta f_\meta(\x) + c$ with $w_\meta, c \in \mathbb{R}$. Hence $\cov{f_\meta(\x), \tilde f_{\meta'}(\x')} = w_{m'} \cov{f_\meta(\x), f_{\meta'}(\x')}$. Together with \cref{as:independent_sources} this becomes $\cov{f_\meta(\x), \tilde f_{\meta'}(\x')} = \delta_{m, m'} w_{m} k_\meta(x, x')$. Using that the covariance is bilinear it follows that
\begin{align}
 k(( \x, \meta), ( \x', t))  
  &= \sum_{\meta' \in \metaset} \cov{f_\meta(\x), \tilde f_{\meta'}(\x')}   
   = w_\meta k_\meta(x, x').
\end{align}
Finally, we have
\begin{align}
 k(( \x, \test), ( \x', \test)) 
  &= \cov{f_\test(x), f_\test(x')} \nonumber \\
  &= k_\test(x, x') + \sum_{\meta \in \metaset} \cov{f_\test(\x), \tilde f_{\meta}(\x')} \quad \text{(\cref{as:strong_correlation} and }k_\test(x, x') = \cov{f_\test(x), \tilde f_\test(x')}) \nonumber \\
  &= k_\test(x, x') + \sum_{\meta \in \metaset} \cov{\tilde f_\test + \sum_{\meta' \in \metaset} \tilde f_{\meta'}(\x), \tilde f_{\meta}(\x')} \quad \text{(\cref{as:strong_correlation})} \nonumber \\
  &= k_\test(x, x') + \sum_{\meta \in \metaset} \cov{\tilde f_\meta(\x), \tilde f_{\meta}(\x')}  \quad \text{(\cref{as:independent_sources}, \cref{as:strong_correlation})}  \nonumber \\
\end{align}
Since $\tilde f_\meta(\x) = w_\meta f_\meta(\x) + c$ as shown above and using bilinearity of the covariance, we have 
\begin{align}
 k(( \x, \test), ( \x', \test)) 
 &= k_\test(x, x') + \sum_{\meta \in \metaset} w_\meta^2 k_\meta(x, x').
\end{align}
By collecting the coefficients corresponding to the $k_\meta$, $k_\test$ we obtain the coregionalization matrices. 

Further, from \cref{lem:coreg_matrices_psd} and \citep{alvarez2012kernels} we obtain that $k_\method^\mathrm{joint}$ is positive semi-definite as a linear combination of positive semi-definite coregionalization matrices and kernels.
That is, \cref{eq:correg_matrices} yield a valid multi-task kernel.

Finally, from the definition of the Dirac-delta and $g_\meta(\nu)$, we have 
\begin{align}
\delta_{\nu=t} \delta_{\nu'=t}
&= \begin{cases}
1&\text{if } \nu=\nu'=t,\\
0&\text{else,}\\ 
\end{cases}
&
g_\meta(\nu)g_\meta(\nu')= 
\begin{cases}
w_\meta^2&\text{if } \nu=\nu'=t,\\
w_\meta&\text{if } \nu=\meta, \nu'=t,\text{or } \nu=t, \nu'=\meta,\\
1&\text{if } \nu=\nu'=\meta,\\
0&\text{else.}
\end{cases}
\end{align}
The statement from the Lemma follows from $\delta_{\nu=t} \delta_{\nu'=t}=[\mathbf{W}_\test]_{(\nu, \nu')}, \, g_\meta(\nu)g_\meta(\nu')=[\mathbf{W}_\meta]_{(\nu, \nu')}$ for all $\nu, \nu'$ and the fact that \cref{eq:correg_matrices} yield a valid kernel.
\qedhere
\end{proof}

\testprior*
\begin{proof}
    According to \cref{eq:our_joint_kernel}, the joint prior model for the meta-observations $\mathbf{y}_\mathrm{meta} = (\mathbf{y}_1, \dots, \mathbf{y}_\mathrm{meta})$ and the test-task function value at a query point $\x$ given by
    \begin{equation}
        \begin{bmatrix}
            \mathbf{y}_\mathrm{meta} \\
            f_\test(\x)
        \end{bmatrix} \sim \mathcal{N} \left(
        \mathbf{0}, \begin{bmatrix}
            \mathbf{K}_\mathrm{meta} & \mathbf{k}_\mathrm{meta} \\
            \mathbf{k}_\mathrm{meta}^\mathrm{T} & \mathbf{K}_\test
        \end{bmatrix}
        \right)
    \end{equation}
    where
    \begin{align*}
    \mathbf{K}_\mathrm{meta} &= \diag\big(k_1(\X_1, \X_1) + \sigma_1^2 \mathbf{I}, \dots, k_\nummeta(\X_\nummeta, \X_\nummeta) + \sigma_\nummeta^2 \mathbf{I} \big), \\
    \mathbf{k}_\mathrm{meta} &= \big( w_1 k_1(\X_1, \x), \dots, w_\nummeta k_\nummeta(\X_\nummeta, \x) \big), \\
    \mathbf{K}_\test &= \big( k_\test(\x, \x) + \sum_\meta w_\meta^2 k_\meta(\x, \x) \big).
    \end{align*}
    Conditioning on the meta-data $\metadata$ yields
    \begin{equation}
     p( f_\test \mid \metadata, \x) = \mathcal{N}\left( m_\method(\x), \Sigma_\method(\x, \x) \right),
    \end{equation}
    where the test-task prior mean $m_\method$ and covariance $\Sigma$ are given by the standard Gaussian conditioning rules
    \begin{align*}
        m_\method(\x) &= \mathbf{k}_\mathrm{meta}^\mathrm{T} \mathbf{K}_\mathrm{meta}^{-1} \mathbf{y}_\mathrm{meta}, \\
        \Sigma_\method(\x, \x) &= \mathbf{K}_\test - \mathbf{k}_\mathrm{meta}^\mathrm{T} \mathbf{K}_\mathrm{meta}^{-1} \mathbf{k}_\mathrm{meta}.
    \end{align*}
    Now since $ \mathbf{K}_\mathrm{meta}$ is block-diagonal, we have
    \begin{equation*}
    \mathbf{K}_\mathrm{meta}^{-1} = \diag\left( (k_1(\X_1, \X_1) + \sigma_1^2 \mathbf{I})^{-1}, \dots, (k_\nummeta(\X_\nummeta, \X_\nummeta) + \sigma_\nummeta^2 \mathbf{I})^{-1} \right),
    \end{equation*}
    so that
    \begin{align*}
        m_\method(\x) 
        &= \sum_\meta w_\meta k_\meta(\x, \X_\meta) ( k_\meta(\X_\meta, \X_\meta) + \sigma_\meta^2 \mathbf{I} )^{-1} \mathbf{y}_\meta, \\
        &= \sum_\meta w_\meta \mu_\meta(\x),
    \end{align*}
    where $\mu_\meta(\x)$ is the per meta-task posterior mean after conditioning on the corresponding data $\mathcal{D}_\meta$. Similarly, for the covariance we have
    \begin{align*}
        \Sigma_\method(\x, \x)
        &= \mathbf{K}_\test - \mathbf{k}_\mathrm{meta}^\mathrm{T} \mathbf{K}_\mathrm{meta}^{-1} \mathbf{k}_\mathrm{meta}, \\
        &= k_\test(\x, \x) + \sum_\meta w_\meta^2 k_\meta(\x, \x) - \sum_\meta w_\meta k_\meta(\x, \X_\meta) ( k_\meta(\X_\meta, \X_\meta) + \sigma_\meta^2 \mathbf{I} )^{-1} w_\meta k_\meta(\X_\meta, \x), \\
        &= k_\test(\x, \x) + \sum_\meta w_\meta^2 k_\meta(\x, \x) - w_\meta^2 k_\meta(\x, \X_\meta) ( k_\meta(\X_\meta, \X_\meta) + \sigma_\meta^2 \mathbf{I} )^{-1} k_\meta(\X_\meta, \x), \\
        &= k_\test(\x, \x) + \sum_\meta w_\meta^2 \big(  k_\meta(\x, \x) - k_\meta(\x, \X_\meta) ( k_\meta(\X_\meta, \X_\meta) + \sigma_\meta^2 \mathbf{I} )^{-1} k_\meta(\X_\meta, \x) \big), \\
        &= k_\test(\x, \x) + \sum_\meta w_\meta^2 \Sigma_\meta(\x, \x),
    \end{align*}
    where $\Sigma_\meta$ is the corresponding per meta-task posterior covariance. \qedhere
    
\end{proof}

\section{Complexity Analysis}
\label{ap:complexity_analysis}
Here we analyze the computational complexity of evaluating the likelihood of the test-data under the prior \cref{eq:test_posterior}. We break this down into the complexities for evaluating the posterior mean and the posterior covariance of the test-task.

As an intermediate step, the kernel matrices of the meta-tasks need to be inverted, which is of complexity $O(\sum_{\meta \in \metaset} N_\meta^3)$. This only needs to happen once and then those inverted matrices can be cached.

To construct the test-task prior we need to evaluate each meta-task posterior mean at all x-values in the test-task dataset $\X_\test$. Given the cached kernel matrices each evaluation costs $O(N_\meta^2)$ multiplications resulting in a complexity of $O(N_\test \sum_{\meta \in \metaset} N_\meta^2)$. In addition, we also need to evaluate the posterior covariance at the same parameters. In contrast to the posterior mean computation, evaluating the covariance matrix at all parameters  $\X_\test$ jointly results in a lower computational complexity than isolated forward calculation of all $N_\test^2$ entries. The corresponding matrix multiplication in \cref{eq:test_posterior} is of order $O(N_\test^2 N_\meta + N_\test N_\meta^2)$. The complexities for all meta-tasks add up to $O(N_\test^2 \sum_{\meta \in \metaset} N_\meta + N_\test \sum_{\meta \in \metaset} N_\meta^2)$.

Further in the likelihood we also need to evaluate the test-task posterior at the test-task data points. This requires to evaluate \cref{eq:test_posterior} and is dominated by the inversion of the test-task kernel matrix, which is of order $O(N_\test^3)$.

Summarizing the complexity of evaluating the likelihood is given by the complexity of inverting the meta-task kernel matrices $O(\sum_{\meta \in \metaset} N_\meta^3)$ once and the reoccurring cost of constructing the test-task prior and evaluating the test-task posterior, which is of order $O(N_\test^2 \sum_{\meta \in \metaset} N_\meta + N_\test \sum_{\meta \in \metaset} N_\meta^2 + N_\test^3)$.

As mentioned in \cref{sec:our_approach}, we can use \cref{as:likelihood_modularization} in order to cache intermediate results in order to further reduce complexity.

\clearpage

\section{Details on Experiments}
\FloatBarrier
\label{ap:experiment_details}
In all experiments we use BoTorch's implementation of Upper Confidence Bound acquisition function with the exploration coefficient $\beta^{1/2} = 3$. For the experiments on the synthetic benchmarks, the meta and test functions are sampled randomly from a function family, while the meta-data parameters are sampled uniformly from the task's domain $D$. We optimize the test function from scratch without initial samples. We add i.i.d. zero-mean Gaussian observational noise during the data generation with a standard deviation of 1.0 for Branin and 0.1 for Hartmann 3D and 6D. This amount of noise corresponds to about one percent of the output scale of the benchmark in the search space.

We implement all GP models using the squared-exponential kernel with automatic relevance determination. The GP hyperparameters are optimized at each BO iteration by maximizing the likelihood of the observed data using the L-BFGS-B optimizer with 5 initial guesses. These guesses are sampled from the prior of the hyperparameter. We normalize the GP data as explained in \cref{ap:data_normalization}. All other details regarding training and prediction are kept fixed to BoTorch's defaults \citep{balandat2020botorch}.

\subsection{Data Normalization}
\label{ap:data_normalization}
We follow the common practice and normalize the GP data. The search space of the input parameters is rescaled to the unit hypercube in the data pre-processing pipeline. The observations are normalized to zero-mean unit-variance individually for each GP. An exception is \ourmethod's test-task GP for which observations, $\mathbf{y}_\test$, are normalized with respect to the mean and variance of $\mathbf{y}_\test \cup \mathbf{y}_{1:M}$. We prefer this strategy over normalizing solely with respect to $\mathbf{y}_\test$ whose statistics are volatile at the start of the optimization. 

\subsection{Implementation Details of \ourmethod}
\label{ap:implementation_details_our_method}
We implement \ourmethod according to \cref{alg:overview}. 
\begin{enumerate}
    \item We train individual GP models on the data of each meta-task and independently optimize the marginal likelihoods. These meta-data are individually normalized to zero-mean unit-variance.
    \item We construct the test-task prior as in \cref{thm:test_prior} by summing over the posteriors of the GP of each meta-task. \cref{as:likelihood_modularization} allows us to cache $\mu_\meta(\mathbf{X}_\test)$ and $\Sigma_\meta(\mathbf{X}_\test, \mathbf{X}_\test)$.
    \item We optimize the hyperparameters, $\theta_\test$, of the test-task GP. For this we normalize $\mathbf{y}_\test$ as explained in \cref{ap:data_normalization}. This joint normalization implies that $\mathbf{y}_\test$ generally have a non-zero mean and non-unit variance.
    \item We condition the test-task GP on $\mathcal{D}_\test$. Since $\mathbf{y}_\test$ are not standardized, the optimum hyperparameters of the test-task kernel, $k_\test$, are distributed differently than those of standard GPs. We consider this when choosing prior distributions for the GP hyperparameters in \cref{ap:prior_hps}.
\end{enumerate}

\subsection{Importance of Satisfying the Assumptions of \ourmethod in Practice}
\label{ap:validity_of_scamlgp_assumptions}
\ourmethod is a modular and flexible method for meta-learning that is scalable in the number of tasks. \ourmethod can be applied to a wide variety of meta-learning settings as demonstrated in \cref{sec:experiments} and is designed to excel in the regime of relatively little amount of data per meta-task.

While the assumptions of \ourmethod may seem restrictive, \cref{as:independent_sources,as:strong_correlation,as:likelihood_modularization}, empirical results clearly demonstrate that \ourmethod performs excellently in most scenarios. This is because \ourmethod has built-in mechanisms that can easily handle violations of these assumptions. In particular,
\begin{itemize}
    \item \cref{as:independent_sources}, independence of the prior distributions of meta-task models, is clearly violated in most use cases. However, they do not limit learning as long as each meta-task contains sufficient data allowing for a probabilistic description of the meta-task function. This amount of data can be surprisingly little for GPs as seen in \cref{fig:ablation_summary} -- as little as 16 points in a six-dimensional search space are enough to significantly improve over GPBO and other meta-learning baselines.
    \item \cref{as:strong_correlation}, perfect correlation between meta-task models and components of the test-task models, is likewise violated in most use cases. \ourmethod can easily handle such violations via the test-task component of the kernel, $k_\test(\mathbf{x}, \mathbf{x}')$ (see \cref{eq:test_prior}), which can learn arbitrary non-linear deviations from the prediction of the meta-task models. In the worst-case scenario in which no linear combination of meta-task functions is useful, \ourmethod simply ignores the meta-data and optimizes the process from scratch using the test-task component of the kernel. Such an example can be seen in \cref{fig:hpo} (FC-Net Protein benchmark).
    \item \cref{as:likelihood_modularization}, independence between meta-task HPs, $\theta_m$, and test-task observations $\mathcal{D}_\test$, is motivated by the meta-learning setting in which meta-data are abundant but test-data are scarce.
\end{itemize}

Only \cref{as:independent_sources} is necessary for making our method scalable. \cref{as:strong_correlation} confers an explainable structure to the kernel and to the test-task prior in \cref{eq:test_prior} but other choices are, in principle, also possible. Finally, \cref{as:likelihood_modularization} leads to a particularly efficient implementation of \ourmethod.

\subsection{Implementation Details of \RMBO}
\label{ap:implementation_details_rmbo}
We implement \RMBO according to the UCB based algorithm in \citep{Dai2022provably} with the following choices. For ease of implementation and to be consistent with other methods, we set the exploration coefficients to a constant value $\beta_t=\tau=3$. To ensure invariance of the acquisition function maximizer under a joint rescaling of all functions, we divide the bounds on the function gap $\bar d_{t, i}$ by the mean of the standard deviations across all meta-tasks before using them in the calculation of the weights $\nu_t$ and $\omega_i$. Finally, we set the hyperparameter $\delta=0.05$.

\subsection{Choices of Prior Distributions and Constraints for GP Hyperparameters}
\label{ap:prior_hps}
In the following we discuss our choice of the prior distribution of the GP hyperparameters as well as the constraints we place on them.

\paragraph{Lengthscale} The prior of the lengthscale hyperparameter, $\theta_l$, is kept fixed to BoTorch's default, $\theta_l \sim \Gamma(3, 6)$, where $\Gamma$ denotes the Gamma distribution. This corresponds to a 5-95\% quantile range of 0.14 -- 1.05.
Exception to this is the test-task kernel of \StackGP and \ourmethod, $k_\test$, with $\log(\theta_l) \sim \mathcal{N}(0.5, 1.5)$ with a 5-95\% quantile range of 0.14 -- 19.44. With this broader distribution we introduce less inductive bias while keeping the risk of over-fitting low owing to the existence of an informative prior from the models of the meta-tasks. The lengthscale parameter is constrained to lie between $10^{-4} - 10^2$ to avoid running into numerical issues. 

\paragraph{Outputscale} The prior of the outputscale hyperparameter (signal variance), $\theta_o$, is also kept fixed to BoTorch's default, ${\theta_o \sim \Gamma(2.0, 0.15)}$ with a 5-95\% quantile range of 2.4 -- 31.8. Exception to this is the test-task kernel of \ourmethod due to the different normalization strategy discussed in \cref{ap:data_normalization}. We choose a less biased distribution to accommodate this change in data normalization, $\log(\theta_o) \sim \mathcal{N}(-2.0, 3.0)$ with a 5-95\% quantile range of $10^{-3} - 18.8$. The outputscale parameter is constrained to lie between $10^{-4} - 10^2$ to avoid running into numerical issues.

\paragraph{Observation noise} The prior of the observational-noise hyperparameter (noise variance), $\theta_n$, is set to $\log(\theta_n) \sim \mathcal{N}(-8, 2)$ with 5-95\% quantile range of $1.25\cdot 10^{-5} - 9\cdot 10^{-3}$. The noise parameter is constrained to lie between $10^{-8} - 10^{-2}$ to avoid running into numerical issues.

\paragraph{Weights of \ourmethod} We choose a generic prior distribution for the weights, ${w_\meta \sim \Gamma(1, 1)}$, independently for each meta-task $\meta$. This distribution is flat and thus unbiased for $w_\meta \lesssim 1$, and decaying quickly for $w_\meta \gtrsim 1$. We constrain the weights to be strictly positive, $w_\meta > 0$, implying that we only learn correlations and ignore anti-correlations.

\subsection{Synthetic Benchmarks}
\label{ap:synthetic_benchmarks}
\FloatBarrier

\paragraph{The Branin Benchmark}

The Branin function is a two-dimensional function with three global optima defined as
\begin{equation}
f(x_1, x_2; a, b, c, r, s, t) = a(x_2-bx_1^2+cx_1-r)+s(1-t)\cos(x_1)+s,\quad x_1\in\left[-5, 10\right], x_2\in\left[0, 15\right]
\label{eq:branin_fun}
\end{equation}
We convert it to a meta-learning benchmark by choosing the following probability distributions for the parameters $(a, b, c, r, s, t)$:
\begin{equation}
\begin{aligned}
a &\sim\mathcal{U}(0.5, 1.5), \\
b &\sim\mathcal{U}(0.1, 0.15), \\
c &\sim\mathcal{U}(1, 2),  \\
r &\sim\mathcal{U}(5, 7), \\
s &\sim\mathcal{U}(8, 12), \\
t &\sim\mathcal{U}(0.03, 0.05).
\end{aligned}
\label{eq:branin_family}
\end{equation}
The Branin meta-learning benchmark is thus defined over a six-dimensional uniform distribution. For generating the data of $n_s$ meta-tasks, we draw $n_s$ random tasks using \cref{eq:branin_family}, and sample a given number of parameters per task uniformly at random.

\paragraph{The Hartmann 3D Benchmark}
The Hartmann~3D function is a sum of four three-dimensional Gaussian distributions and is defined by
\begin{equation}
f(\mathbf{x}; \bm{\alpha}) = -\sum_{i=1}^4\alpha_i\exp\left(-\sum_{j=1}^3 A_{i,j}\left(x_j-P_{i,j}\right)^2 \right),\quad \mathbf{x} \in \left[0, 1\right]^3,
\label{eq:hartmann3_fun}
\end{equation}
with
\begin{equation*}
\mathbf{A} = \begin{bmatrix}
3.0 &10  &30 \\ 
0.1 &10  &35 \\ 
3.0 &10  &30 \\ 
0.1 &10  &35 
\end{bmatrix},\quad \mathbf{P} = 10^{-4}
\begin{bmatrix}
3689 &1170  &2673 \\ 
4699 &4387  &7470 \\ 
1091 &8732  &5547 \\ 
381  &5743  &8828 
\end{bmatrix}.
\end{equation*}
The original Hartmann~3D function is given by $\bm{\alpha}=(1.0, 1.2, 3.0, 3.2)^T$. In this paper, a family of functions is formed by choosing the following probability distributions for the parameters $\bm{\alpha}=(\alpha_1, \alpha_2, \alpha_3, \alpha_4)^T$:
\begin{equation}
\alpha_1\sim\mathcal{U}(1.00, 1.02), \quad
\alpha_2\sim\mathcal{U}(1.18, 1.20), \quad
\alpha_3\sim\mathcal{U}(2.8, 3.0), \quad
\alpha_4\sim\mathcal{U}(3.2, 3.4).
\label{eq:hartmann3_family}
\end{equation}
The Hartmann~3D family therefore spans a four-dimensional uniform distribution. For generating the data of $\nummeta$ meta-tasks, we draw $\nummeta$ random tasks using \cref{eq:hartmann3_family}, and sample a given number of parameters per task uniformly at random.

\paragraph{The Hartmann 6D Benchmark}
The Hartmann 6D function is a sum of four six-dimensional Gaussian distributions with six local optima and one global optimum and is defined by
\begin{equation}
f(\mathbf{x}; \bm{\alpha}) = -\sum_{i=1}^4\alpha_i\exp\left(-\sum_{j=1}^6 A_{i,j}\left(x_j-P_{i,j}\right)^2 \right),\quad \mathbf{x} \in \left[0, 1\right]^6,
\label{eq:hartmann6_fun}
\end{equation}
with
\begin{equation*}
\begin{split}
\mathbf{A} &= \hphantom{10^{-4}}
\begin{bmatrix}
10   & 3   & 17   & 3.5 & 1.7 & 8 \\
0.05 & 10  & 17   & 0.1 & 8   & 14 \\
3    & 3.5 & 1.7  & 10  & 17  & 8 \\
17   & 8   & 0.05 & 10  & 0.1 & 14
\end{bmatrix},\\
\mathbf{P} &= 10^{-4}
\begin{bmatrix}
1312 & 1696 & 5569 & 124  & 8283 & 5886 \\
2329 & 4135 & 8307 & 3736 & 1004 & 9991 \\
2348 & 1451 & 3522 & 2883 & 3047 & 6650 \\
4047 & 8828 & 8732 & 5743 & 1091 & 381
\end{bmatrix}.
\end{split}
\end{equation*}
The original Hartmann~6D function is given by $\bm{\alpha}=(1.0, 1.2, 3.0, 3.2)^T$. We convert it to a meta-learning benchmark by placing the following probability distributions for $\bm{\alpha}=(\alpha_1, \alpha_2, \alpha_3, \alpha_4)^T$ based on the emukit~\citep{emukit2019} implementation\footnote{\url{https://web.archive.org/web/20230126095651/https://github.com/EmuKit/emukit/blob/b4e59d0867c3a36b72451e7ec5864491d3c11bbe/emukit/test_functions/multi_fidelity/hartmann.py}}:
\begin{equation}
\alpha_1\sim\mathcal{U}(1.00, 1.02), \quad
\alpha_2\sim\mathcal{U}(1.18, 1.20), \quad
\alpha_3\sim\mathcal{U}(2.8, 3.0), \quad
\alpha_4\sim\mathcal{U}(3.2, 3.4).
\label{eq:hartmann6_family}
\end{equation}
The Hartmann~6D meta-learning benchmark is therefore defined over a four-dimensional uniform distribution. For generating the data of $M$ meta-data sets, we draw $M$ random tasks using \cref{eq:hartmann6_family}, and sample a given number of parameters per meta-task uniformly at random.

\FloatBarrier
\subsection{Meta-Learning and Bayesian Optimization on Tabular Benchmarks}
\label{ap:bo_on_tabular_benchmarks}
\FloatBarrier

For the machine learning hyperparameter optimization benchmarks we use tabular benchmarks from the HPOBench framework \citep{eggensperger2021hpobench} from the Git branch "master" at commit 47bf141 (licensed under Apache-2.0). The objective values reported to each optimizer are average objective values over all available seeds in the look up table for each configuration respectively, i.e. no specific seed is given to HPOBench's objective function.
All search spaces consist of ordinal parameters due to the tabular nature of the benchmark and do not contain any conditions, i.e., levels hierarchy. We carry out meta-learning and Bayesian optimization on these tabular benchmarks according to \cref{alg:bo_on_tabular_benchmarks}.

\begin{algorithm}
   \caption{Meta-Learning and Bayesian Optimization on Tabular Benchmarks}
   \label{alg:bo_on_tabular_benchmarks}
\begin{algorithmic}[1]
   \STATE {\bfseries Input:} budget $T$, tabular data $\mathcal{D}=\{X, Y\}$, task-ID set $\mathcal{I}$, num. parameters per meta-task $N_\meta$
   \STATE Sample meta-tasks, $\mathcal{I}_\meta$, $\meta \in \mathcal{M} = \{1, 2, \ldots, M\}$, and test-task, $\mathcal{I}_\test$, IDs from $\mathcal{I}$.
   \STATE Generate the meta-data, $\mathcal{D}_\meta = \mathcal{D} \cap \mathcal{I}_\meta$, and test-data, $\mathcal{D}_\test = \mathcal{D} \cap \mathcal{I}_\test$.
   \STATE For each meta-task, randomly sub-sample the meta-data to obtain $N_\meta$ points per meta-task.
   \STATE Condition the model on the meta-data $\mathcal{D}_{1:M}$.
   \STATE $\mathcal{D}_\test = \emptyset$
   \FOR{$i\in \{1,\ldots,T\}$}
      \STATE Calculate $\mathbf{x}_i$ with \cref{eq:bo} by maximizing the acquisition function over the discrete search space.
      \STATE Optimize likelihood function and condition the model on $\mathcal{D}_\test \leftarrow \mathcal{D}_\test \cup \{\mathbf{x}_i, y_i(\mathbf{x}_i)\}$ as in \cref{eq:test_posterior}.
      \ENDFOR
\end{algorithmic}
\end{algorithm}

Each of the following subsections contains the details for the respective task families.

\FloatBarrier
\subsubsection{Random Forest (RF)}
\FloatBarrier

The (meta-)tasks were sampled from the following HPOBench task IDs: 10101, 12, 146195, 146212, 146606, 146818, 146821, 146822, 14965, 167119, 167120, 168329, 168330, 168331, 168335, 168868, 168908, 168910, 168911, 168912, 3, 31, 3917, 53, 7592, 9952, 9977, 9981

The tabular data version used by HPOBench was \citep{Mallik2021RFv3}.

We fixed the benchmark's fidelities to the default fidelities provided by HPOBench since the multi-fidelity scenario was not considered in our experiments: $n\_estimators = 512$, $subsample = 1$

The search space contained the following parameters with their respective ordinal values (truncated after the third decimal place for readability).

\begin{tabular}{l | l}
 Name & Values \\
 \toprule
max\_depth & 1.0, 2.0, 3.0, 5.0, 8.0, 13.0, 20.0, 32.0, 50.0 \\
 \midrule
max\_features & 0.0, 0.111, 0.222, 0.333, 0.444, 0.555, 0.666, 0.777, 0.888, 1.0 \\
 \midrule
min\_samples\_leaf & 1.0, 3.0, 5.0, 7.0, 9.0, 11.0, 13.0, 15.0, 17.0, 20.0 \\
 \midrule
min\_samples\_split & 2.0, 3.0, 5.0, 8.0, 12.0, 20.0, 32.0, 50.0, 80.0, 128.0 \\
\end{tabular}

\subsubsection{Logistic Regression (LR)}
\FloatBarrier

The (meta-)tasks were sampled from the following HPOBench task IDs: 10101, 146195, 146606, 146821, 14965, 167120, 168330, 168335, 168908, 168910, 168912, 31, 53, 9952, 9981, 12, 146212, 146818, 146822, 167119, 168329, 168331, 168868, 168909, 168911, 3, 3917, 7592, 9977

The tabular data version used by HPOBench was \citep{Mallik2021LRv3}.

We fixed the benchmark's fidelities to the default fidelities provided by HPOBench since the multi-fidelity scenario was not considered in our experiments: $\mathrm{iter} = 1000$, $\mathrm{subsample} = 1.0$

The search space contained the following parameters with their respective ordinal values (truncated after the sixth decimal place for readability).

\begin{table}[h]
\begin{tabular}{l | p{12cm}}
 Name & Values \\
 \toprule
alpha & 0.000009, 0.000016, 0.000026, 0.000042, 0.000068, 0.000110, 0.000177, 0.000287, 0.000464, 0.000749, 0.001211, 0.001957, 0.003162, 0.005108, 0.008254, 0.013335, 0.021544, 0.034807, 0.056234, 0.090851, 0.146779, 0.237137, 0.383118, 0.618965, 1.0 \\
 \midrule
eta0 & 0.000009, 0.000016, 0.000026, 0.000042, 0.000068, 0.000110, 0.000177, 0.000287, 0.000464, 0.000749, 0.001211, 0.001957, 0.003162, 0.005108, 0.008254, 0.013335, 0.021544, 0.034807, 0.056234, 0.090851, 0.146779, 0.237137, 0.383118, 0.618965, 1.0 \\
\end{tabular}
\end{table}

\FloatBarrier
\subsubsection{Support Vector Machine (SVM)}
\FloatBarrier

The (meta-)tasks were sampled from the following HPOBench task IDs: 10101, 146195, 146606, 146821, 14965, 167120, 168330, 168335, 168908, 168910, 168912, 31, 53, 9952, 9981, 12, 146212, 146818, 146822, 167119, 168329, 168331, 168868, 168909, 168911, 3, 3917, 7592, 9977

The tabular data version used by HPOBench was \citet{Mallik2021SVMv3}.

We fixed the benchmark's fidelity to the default fidelity provided by HPOBench since the multi-fidelity scenario was not considered in our experiments: $subsample = 1.0$

The search space contained the following parameters with their respective ordinal values (truncated after the fourth decimal place for readability).

\begin{table}[h]
\begin{tabular}{l | p{12cm}}
 Name & Values \\
 \toprule
C & 0.0009, 0.0019, 0.0039, 0.0078, 0.0156, 0.0312, 0.0625, 0.125, 0.25, 0.5, 1.0, 2.0, 4.0, 8.0, 16.0, 32.0, 64.0, 128.0, 256.0, 512.0, 1024.0 \\
 \midrule
gamma & 0.0009, 0.0019, 0.0039, 0.0078, 0.0156, 0.0312, 0.0625, 0.125, 0.25, 0.5, 1.0, 2.0, 4.0, 8.0, 16.0, 32.0, 64.0, 128.0, 256.0, 512.0, 1024.0 \\
\end{tabular}
\end{table}

\FloatBarrier
\subsubsection{XGBoost (XGB)}
\FloatBarrier

The (meta-)tasks were sampled from the following HPOBench task IDs: 10101, 12, 146212, 146606, 146818, 146821, 146822, 14965, 167119, 167120, 168911, 168912, 3, 31, 3917, 53, 7592, 9952, 9977, 9981

The tabular data version used by HPOBench was \citet{Mallik2021XGBv3}.

We fixed the benchmark's fidelities to the default fidelities provided by HPOBench since the multi-fidelity scenario was not considered in our experiments: $n\_estimators = 2000, subsample = 1$

The search space contained the following parameters with their respective ordinal values (truncated after the fourth decimal place for readability).

\begin{table}[h]
\begin{tabular}{l | p{10cm}}
 Name & Values \\
 \toprule
colsample\_bytree & 0.1, 0.2, 0.3, 0.4, 0.5, 0.6, 0.7, 0.8, 0.9, 1.0 \\
 \midrule
eta & 0.0009, 0.0021, 0.0045, 0.0098, 0.0212, 0.0459, 0.0992, 0.2143, 0.4629, 1.0 \\
 \midrule
max\_depth & 1.0, 2.0, 3.0, 5.0, 8.0, 13.0, 20.0, 32.0, 50.0 \\
 \midrule
reg\_lambda & 0.0009, 0.0045, 0.0212, 0.0992, 0.4629, 2.1601, 10.0793, 47.0315, 219.4544, 1024.0 \\
\end{tabular}
\end{table}

\FloatBarrier
\subsubsection{FC-Net}
\FloatBarrier

We ran the following scenarios: Slice Localization (Slice), Protein Structure (Protein), Naval Propulsion (Naval), Parkinson's Telemonitoring (Parkinson's). We used these in a leave-one-out meta-learning setup, where one is the test-task and the other three the meta-tasks. To keep the search space non-hierarchical, we fixed the two activation function parameters to "relu" and the learning rate schedule to "cosine", which corresponds to the most optimal value across all benchmarks and yields the following, fully ordinal search space.

\begin{table}[h]
\begin{tabular}{l | p{10cm}}
 Name & Values \\
 \toprule
batch\_size & 8, 16, 32, 64 \\
 \midrule
dropout\_1 & 0.0, 0.3, 0.6 \\
 \midrule
dropout\_2 & 0.0, 0.3, 0.6 \\
 \midrule
init\_lr & 0.0005, 0.001, 0.005, 0.01, 0.05, 0.1 \\
\midrule
n\_units\_1 & 16, 32, 64, 128, 256, 512 \\
\midrule
n\_units\_2 & 16, 32, 64, 128, 256, 512 \\
\end{tabular}
\end{table}

The fidelity was fixed to the maximum value (usually 100 epochs) and the average for objective values across all available seeds was used.

\FloatBarrier
\subsection{Computational Resources}
\label{ap:computational_resources}
We conducted our experiments on Azure's Standard\_D64s\_v3 instances, which offer 64 virtual cores and are powered by "Intel(R) Xeon(R) CPU E5-2673 v4". All experiments and results shown in the paper and appendix took approximately 16.5 days worth of sequential compute in total.

\clearpage

\FloatBarrier
\section{Additional Experiments}
\label{ap:ablation}

We provide the simple-regret plots of the ablation study discussed in \cref{fig:ablation_summary}. In this study we explore the performance of the different baselines on the synthetic benchmarks for various configurations in the meta-data.

\subsection{Ablation studies on the Branin benchmark}
\FloatBarrier

\begin{figure}[h]
    \centering
    \includegraphics[width=\textwidth]{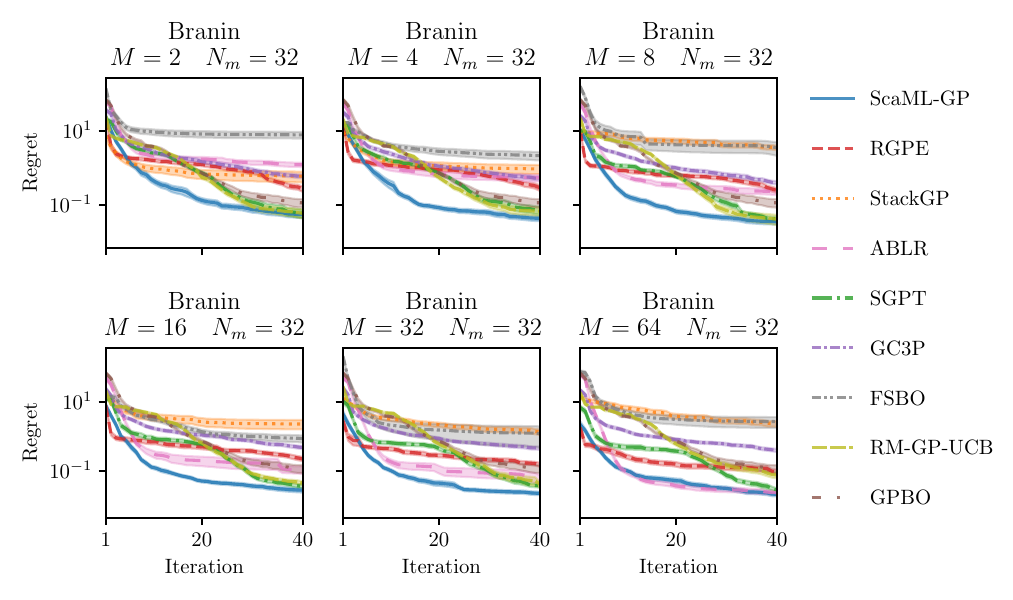}
    \caption{
        Experimental results on the Branin benchmark for different number of meta-tasks. Each meta-task is endowed with 32 points sampled uniformly at random from the task function's domain $D$.
        }
    \label{fig:ablation_branin_num_meta_tasks}
\end{figure}

\begin{figure}[h]
    \centering
    \includegraphics[width=\textwidth]{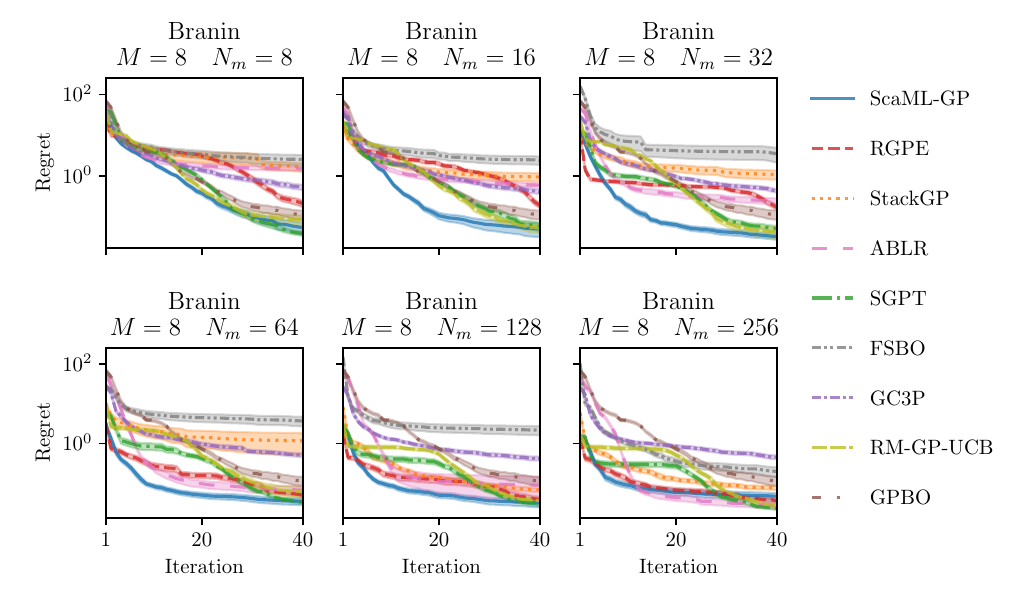}
    \caption{
        Experimental results on the Branin benchmark for different number of points per task, which are sampled uniformly at random from the task function's domain $D$. Here, we keep the number of meta-tasks fixed to eight.
        }
    \label{fig:ablation_branin_num_points_per_task}
\end{figure}
\clearpage

\FloatBarrier
\subsection{Ablation studies on the Hartmann6 benchmark}
\FloatBarrier

\begin{figure}[h]
    \centering
    \includegraphics[width=\textwidth]{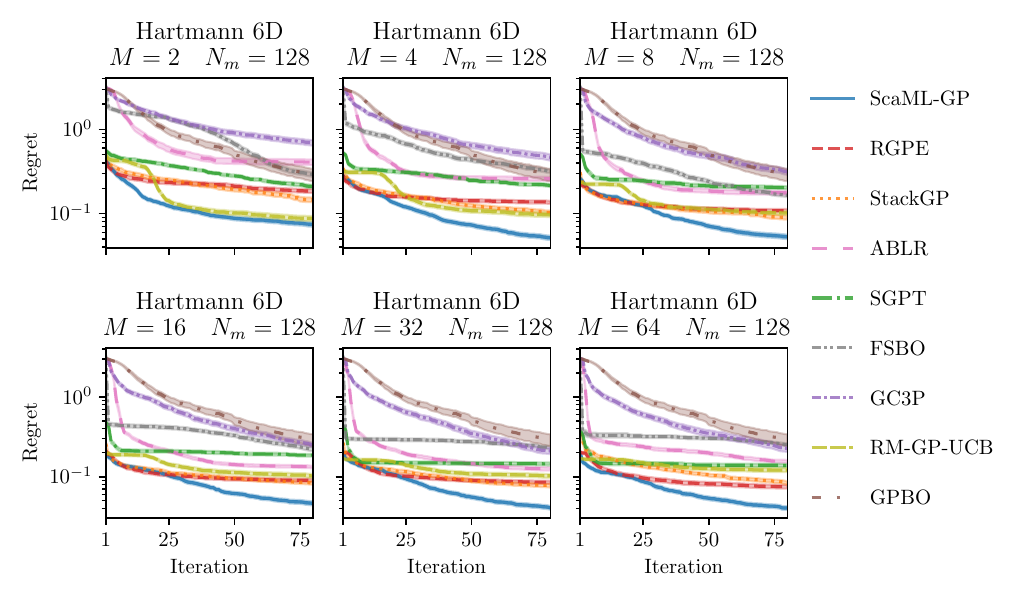}
    \caption{
        Experimental results on the Hartmann6 benchmark for different number of meta-tasks. Each meta-task is endowed with 128 points sampled uniformly at random from the task function's domain $D$.
        }
    \label{fig:ablation_hartmann6_num_meta_tasks}
\end{figure}

\begin{figure}
    \centering
    \includegraphics[width=\textwidth]{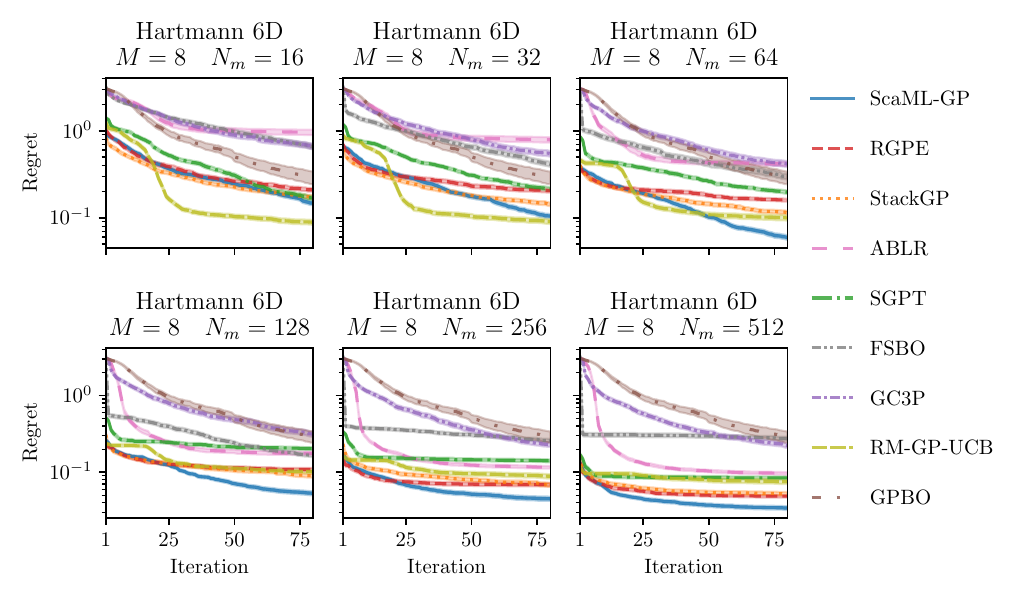}
    \caption{
        Experimental results on the Hartmann6 benchmark for different number of points per task, which are sampled uniformly at random from the task function's domain $D$. Here, we keep the number of meta-tasks fixed to eight.
        }
    \label{fig:ablation_hartmann6_num_points_per_task}
\end{figure}

\end{document}